\documentclass[lettersize,journal]{IEEEtran}
\usepackage{amsmath,amsfonts}
\usepackage{algorithmic}
\usepackage{algorithm}
\usepackage{array}
\usepackage[caption=false,font=normalsize,labelfont=sf,textfont=sf]{subfig}
\usepackage{textcomp}
\usepackage{stfloats}
\usepackage{url}
\usepackage{verbatim}
\usepackage{graphicx}
\usepackage{cite}

\usepackage{booktabs}
\usepackage{nameref}

\usepackage{etoolbox}
\makeatletter
\patchcmd{\@makecaption}
{\scshape}
{}
{}
{}
\makeatletter
\patchcmd{\@makecaption}
{\\}
{.\ }
{}
{}
\makeatother

\usepackage{tikz}
\usepackage[T1]{fontenc}
\usepackage{amsmath,amssymb,amsfonts,mathrsfs,bm}
\usepackage{mathtools}
\usepackage{amsthm}
\usepackage[shortlabels]{enumitem}
\usepackage{graphicx}
\usepackage{epstopdf}
\usepackage{url}
\usepackage{colortbl}
\usepackage{multirow}
\usepackage{xcolor}
\usepackage[normalem]{ulem}
\usepackage{array}
\newcolumntype{L}[1]{>{\raggedright\let\newline\\\arraybackslash\hspace{0pt}}m{#1}}
\newcolumntype{C}[1]{>{\centering\let\newline\\\arraybackslash\hspace{0pt}}m{#1}}
\newcolumntype{R}[1]{>{\raggedleft\let\newline\\\arraybackslash\hspace{0pt}}m{#1}}

\makeatletter
\let\MYcaption\@makecaption
\makeatother
\makeatletter
\let\@makecaption\MYcaption
\makeatother

\usepackage{xparse}



\usepackage{glossaries}
\newacronym{wrt}{w.r.t.}{with respect to}
\newacronym{RHS}{R.H.S.}{right-hand side}
\newacronym{LHS}{L.H.S.}{left-hand side}
\newacronym{iid}{i.i.d.}{independent and identically distributed}

\usepackage{float}

\usepackage[capitalize]{cleveref}
\crefname{equation}{}{}
\Crefname{equation}{}{}
\crefname{claim}{claim}{claims}
\crefname{step}{step}{steps}
\crefname{line}{line}{lines}
\crefname{dmath}{}{}
\crefname{dseries}{}{}
\crefname{dgroup}{}{}
\crefname{Theorem}{Theorem}{Theorems}
\crefname{Corollary}{Corollary}{Corollaries}
\crefname{Proposition}{Proposition}{Propositions}
\crefname{Lemma}{Lemma}{Lemmas}
\crefname{Definition}{Definition}{Definitions}
\crefname{Example}{Example}{Examples}
\crefname{Assumption}{Assumption}{Assumptions}
\crefname{Remark}{Remark}{Remarks}
\crefname{Rem}{Remark}{Remarks}
\crefname{remarks}{Remarks}{Remarks}
\crefname{Theorem_A}{Theorem}{Theorems}
\crefname{Corollary_A}{Corollary}{Corollaries}
\crefname{Proposition_A}{Proposition}{Propositions}
\crefname{Lemma_A}{Lemma}{Lemmas}
\crefname{Definition_A}{Definition}{Definitions}

\usepackage{algorithm}

\ifx\loadbreqn\undefined
	\relax
\else
	\usepackage{breqn} 
\fi

\ifx\renewtheorem\undefined
\ifx\useTheoremCounter\undefined
\newtheorem{Theorem}{Theorem}
\newtheorem{Corollary}{Corollary}
\newtheorem{Proposition}{Proposition}
\newtheorem{Lemma}{Lemma}
\else
\newtheorem{Theorem}{Theorem}

\fi

\newtheorem{Definition}{Definition}

\newtheorem{Assumption}{Assumption}

\fi

\theoremstyle{remark}

\theoremstyle{plain}

\DeclarePairedDelimiter\abs{\lvert}{\rvert}

\newcommand{\qednew}{\nobreak \ifvmode \relax \else
      \ifdim\lastskip<1.5em \hskip-\lastskip
      \hskip1.5em plus0em minus0.5em \fi \nobreak
      \vrule height0.75em width0.5em depth0.25em\fi}

\newcommand{\cond}[2]{\left. {#1}\, \middle| \, {#2} \right.}

\DeclareDocumentCommand \P { g d() g } {%
	\IfNoValueTF {#3} 
	{%
		\IfNoValueTF {#1} 
		{%
			\IfNoValueTF {#2}
			{%
				\mathbb{P}%
			}%
			{%
				\mathbb{P}\left({#2}\right)%
			}%
		}%
		{%
			\IfNoValueTF {#2}
			{%
				\mathbb{P}_{#1}%
			}%
			{%
				\mathbb{P}_{#1}\left({#2}\right)%
			}%
		}%
	}%
	{%
		\IfNoValueTF {#1} 
		{%
			\mathbb{P}\left(\cond{#2}{#3}\right)%
		}%
		{%
			\mathbb{P}_{#1}\left(\cond{#2}{#3}\right)%
		}%
	}%
}

\DeclareDocumentCommand \E { g o g } {%
	\IfNoValueTF {#3} 
	{%
		\IfNoValueTF {#1} 
		{%
			\IfNoValueTF {#2}
			{%
				\mathbb{E}%
			}%
			{%
				\mathbb{E}\left[{#2}\right]%
			}%
		}%
		{%
			\IfNoValueTF {#2}
			{%
				\mathbb{E}_{#1}%
			}%
			{%
				\mathbb{E}_{#1}\left[{#2}\right]%
			}%
		}%
	}%
	{%
		\IfNoValueTF {#1} 
		{%
			\mathbb{E}\left[\cond{#2}{#3}\right]%
		}%
		{%
			\mathbb{E}_{#1}\left[\cond{#2}{#3}\right]%
		}%
	}%
}

\definecolor{gray90}{gray}{0.9}

\ifx\nohighlights\undefined

	\newcommand{\msout}[1]{\text{\color{green} \sout{\ensuremath{#1}}}}
	\newcommand{\del}[1]{{\color{green}\ifmmode \msout{#1}\else\sout{#1}\fi}}
\else

	\newcommand{\msout}[1]{#1}
	\newcommand{\del}[1]{#1}
\fi

\newcommand{\hide}[1]{}

\graphicspath{{./Figures/}} 

\ifx\diagnoselabel\undefined
	\relax
\else
	\makeatletter
	 \def\@testdef #1#2#3{%
		 \def\reserved@a{#3}\expandafter \ifx \csname #1@#2\endcsname
		\reserved@a  \else
	 \typeout{^^Jlabel #2 changed:^^J%
	 \meaning\reserved@a^^J%
	 \expandafter\meaning\csname #1@#2\endcsname^^J}%
	 \@tempswatrue \fi}
	\makeatother
\fi

\usepackage{calc}

\usepackage{geometry}
\usepackage{graphicx}

\usetikzlibrary{arrows.meta, shapes.geometric}

\usepackage{tikz}

\usetikzlibrary{calc}

\usepackage{tikz}
\usepackage{pgfplots}
\pgfplotsset{compat=1.16}
\usepackage{multirow}
\usepackage{array}
\usepgfplotslibrary{groupplots}

\usepackage{breqn}

\begin{document}

\title{FRGNN: Mitigating the Impact of Distribution Shift on Graph Neural Networks via Test-Time Feature Reconstruction}

\author{Rui Ding, Jielong Yang, Ji Feng, Xionghu Zhong, Linbo Xie
	}



\maketitle

\begin{abstract}
	Due to inappropriate sample selection and limited training data, a distribution shift often exists between the training and test sets. This shift can adversely affect the test performance of Graph Neural Networks (GNNs). Existing approaches mitigate this issue by either enhancing the robustness of GNNs to distribution shift or reducing the shift itself. However, both approaches necessitate retraining the model, which becomes unfeasible when the model structure and parameters are inaccessible. To address this challenge, we propose FR-GNN, a general framework for GNNs to conduct feature reconstruction. FRGNN constructs a mapping relationship between the output and input of a well-trained GNN to obtain class representative embeddings and then uses these embeddings to reconstruct the features of labeled nodes. These reconstructed features are then incorporated into the message passing mechanism of GNNs to influence the predictions of unlabeled nodes at test time. Notably, the reconstructed node features can be directly utilized for testing the well-trained model, effectively reducing the distribution shift and leading to improved test performance. This remarkable achievement is attained without any modifications to the model structure or parameters. We provide theoretical guarantees for the effectiveness of our framework. Furthermore, we conduct comprehensive experiments on various public datasets. The experimental results demonstrate the superior performance of FRGNN in comparison to multiple categories of baseline methods.
\end{abstract}


\section{Introduction}
\IEEEPARstart{N}{ode} classification task is of paramount significance in numerous research domains, such as social  networks~\cite{c:32, c:33}, recommendation systems~\cite{c:11, c:12}, fraud detection~\cite{c:31, c:16,c:17}, and fault diagnosis~\cite{c:30, c:14, c:15}. Due to its remarkable ability in processing graph-structured data, Graph Neural Networks (GNNs) have been extensively employed for node classification tasks~\cite{c:18}. At their core, GNNs operate on a message-passing mechanism, transmitting node feature information to neighboring nodes.

However, due to inappropriate training sample selection~\cite{c:20} and the limited availability of training samples~\cite{c:19}, there exists a distribution shift between training and test nodes in graph-structured data. This shift seriously affects the performance of GNNs~\cite{c:7}. Thus, mitigating the impact of distribution shift on classification performance is a crucial issue.

A considerable amount of research has investigated the effects of distribution shift on node classification performance of GNNs~\cite{c:7}. A notable direction among these efforts is enhancing model robustness to improve the model adaptability to the shift. LGD-GNN~\cite{c:9} proposed a decoupled Graph Neural Network model. This model employs a neighbor routing mechanism to obtain different representations in the latent space. By minimizing the correlations among these representations, the method derives decoupled representations. This method extracts representations from graph data that are optimal for classification, enhancing the generalization capability of GNNs. Following a similar motivation, GCN-DVD ~\cite{c:10} utilizes causal inference. By introducing a de-correlation regularization layer, GCN-DVD effectively removes spurious correlations (shortcuts) found in the training set, leading to improved prediction stability in the test phase. WT-AWP ~\cite{c:21} incorporates adversarial weight perturbation into GNN training and proposes the weight truncated algorithm to address the vanishing gradient issue. By minimizing the loss under the worst-case weight perturbation, WT-AWP reduces sensitivity to input variations, enhancing the robustness of model. Although the aforementioned methods can mitigate the negative effects of distribution shift on node classification tasks, they fail to effectively diminish the bias itself. SR-GNN~\cite{c:20} addresses the distribution shift through regularization of the hidden layers of the standard GNN model. The core idea of SR-GNN is to minimize the feature distribution of biased nodes and the feature distribution of independently and identically distributed nodes, thereby enabling the model to learn environment-invariant graph representations.

In light of the aforementioned limitations, various studies have explored the adoption of data augmentation techniques to enrich the diversity of training instances. This approach seeks to directly address the challenge of distribution shift. GAUG~\cite{c:22} employs an edge predictor to modify the relationships between nodes. Such modifications aim to augment the training data from a graph structure perspective. Consequently, this method boosts the generalization performance of GNNs. Building on GAUG, MH-AUG~\cite{c:23} addresses the issue of uncontrollable augmentation intensity (degree of structural change) found in GAUG, offering flexible control over the augmentation strength suitable for various datasets. KDGA~\cite{c:24} points out that both GAUG and MH-AUG might lead to the negative augmentation problem and introduces a knowledge distillation method for data augmentation. KDGA trains a teacher model with augmented data and transfers its knowledge to a student model. The student model is then tested on the original dataset to mitigate the impact of the negative augmentation problem. LA-GNN~\cite{c:19} augments graph data from the perspective of node features. This method employs generative models to learn the feature distribution of neighboring nodes given a central node. The generated node features are then appended to the original node features to facilitate data augmentation. The methods previously described necessitate retraining of models with augmented data to identify a potentially optimal model constructures or set of parameters. However, when the graph data and trained parameters are obscured due to confidentiality concerns, or when the extensive data volume renders retraining computationally onerous, such approaches become infeasible. Hence, mitigating the adverse effects of distribution shift without altering the model structure or its trained parameters remains a challenge.

To mitigate the adverse effects of distribution shift without modifying the given model structure or parameters, adjusting node features during the test phase to make the embeddings of test nodes more similar to those of training nodes of the same class emerges as a viable method. However, this method faces three primary challenges: (1) How to modify node features so that the embeddings of test nodes closely resemble those of training nodes within the same class; (2) How to mitigate the shift while maintaining the original semantics of the node features; (3) How to theoretically validate the effectiveness of this method.

To address these challenges, we conducted an in-depth analysis of the characteristics of GNNs of the inductive node classification task. In our pursuit, we emphasize the importance of strategically manipulating the features of labeled nodes during the test phase. By doing so, we aspire to ensure that, subsequent to the message passing procedure, the embeddings of test nodes draw closer in similarity to the embeddings of training nodes from the same class. Such alignment of embeddings can effectively facilitate the classification of test nodes. To achieve this goal, we propose a general framework for GNNs to conduct feature reconstruction called FR-GNN. In this framework, the modified features for labeled nodes are derived by mapping the actual labels back into the feature space. We also dive into its theoretical grounding. Specifically, we provide theoretical assurances for the validity of this framework. The following are the primary contributions of this paper:

\begin{itemize}
	\item To tackle Challenge 1, we propose a feature reconstruction framework for GNNs. This framework is adaptable to a broad range of GNN methodologies without necessitating modifications to the underlying GNN architecture. Additionally, there is no requirement for adjusting the trained parameters, ensuring its versatility and applicability across different contexts. Such a design ensures its flexibility and portability.
	\item For Challenge 2, our proposed framework only modifies the features of labeled nodes during feature reconstruction, without altering the original semantics of the node features. This approach ensures the interpretability of the model.
	\item In response to challenge 3, our analysis elucidates that by substituting the features of labeled nodes with class representative embeddings, the discrepancy between the embeddings of test nodes and those of training nodes from an identical class is diminished.
\end{itemize}

\section{Method}
    In this section, we introduce our proposed feature reconstruction framework for GNNs. Initially, we formally define the issue of reconstructing node features to minimize the detrimental impact on node classification caused by the distribution shift between test node embeddings and the embeddings of training nodes belonging to the same class. We then delve into the sub-problems inherent to this issue, namely: What kind of node features should replace the original ones to reduce such embeddings bias? And, how do we identify such node features for feature reconstruction? Subsequently, we provide a theoretical proof that class representative embeddings can meet the requirements. Based on the theoretical analysis, we outline the specific implementation method for this framework.

\subsection{Problem Statement}
    GNNs are negatively impacted by distribution shift during test. After undergoing message passing, when the embeddings of training nodes of the same class differ significantly from the test node embeddings, the performance of node classification can be severely impaired. Therefore, mitigating the distribution shift is equivalent to diminishing the discrepancy between embeddings of labeled nodes and unlabeled nodes of the same class.

    To reduce the embeddings bias without altering the model structure or trained parameters, reconstructing node features is a viable approach. By altering the features of the labeled nodes, the embeddings of test nodes adapt accordingly upon aggregation due to the inherent message passing mechanism. We aim to diminish the distance between the embeddings of test nodes after feature reconstruction and the original embeddings of the labeled nodes used for training. To discribe this distance, provide a definition for the metric of embeddings bias.

    \begin{Definition}
        \textbf{Graph embeddings Bias(GEB)}. Let's denote the embeddings matrix of labeled nodes used for training as $H_{train}$. The embeddings of test nodes, related to the node feature matrix $X$, can be symbolized as $\Phi(X)$, where $\Phi(\cdot)$ denotes message passing function. The set of node labels is denoted as $\mathcal{C}$. The training and test nodes belonging to the $c^{th}$ class are $\mathcal{N}_c$ and $\mathcal{M}_c$ respectively. The GEB is then defined as:
        \begin{equation}
            D(X)=\sum_{c\in \mathcal{C}}\sum_{i\in \mathcal{M}_c}\min _{j\in \mathcal{N}_c} \lVert \Phi(X)[i]-H_{train}[j]\rVert_2
        \end{equation}
    \end{Definition}
    This metric represents the summation of the distances between each element in the test embeddings and its closest counterpart in the training embeddings. We assume that the nearest training embeddings to a given test embeddings falls under the same category. A larger value of this metric indicates a more pronounced bias between the training and test embeddings from the same class. Consequently, a graph neural network trained on the training set is more likely to underperform when evaluated on the test set.

    Specifically, the GEB is a function of the node feature matrix $X$. Our objective is to identify a new matrix $X^*$ to replace $X$, reducing the GEB. Mathematically, this can be represented as:
    \begin{equation}
        D(X^*)\leq D(X)
    \end{equation}

    To address this issue, two key questions need to be clarified:

    \begin{itemize}
        \item What conditions should $X^*$ satisfy? \label{Q1}
        \item How can $X^*$ be found without altering the well-trained GNN model? \label{Q2}
    \end{itemize}

    To address the first question, we need to clarify the characteristics of $X^*$ that can achieve the objective from a theoretical perspective. To tackle the second question, we need to identify a practical approach for computing $X^*$, while keeping the GNN model structure and trained parameters unchanged. We will now elaborate on these two issues.

\subsection{What conditions should $X^*$ satisfy?}
    In this section, we have conducted a theoretical analysis on the conditions that $X^*$ should satisfy. We first postulate a conjecture for $X^*$ under ideal conditions. Based on this conjecture, we define a class representative embeddings for constructing $X^*$. Subsequently, we theoretically prove the relationship between the embeddings of test nodes and the embedding expectation of their respective classes. Additionally, we elucidate the connection between the embedding expectation and the class representative embeddings.Leveraging these insights, we demonstrate that employing $X^*$ as a substitute for the original feature matrix $X$ during test can mitigate the bias between the embeddings of test nodes and training nodes.

    \begin{figure}[]
        \centering
        \includegraphics[width=1\columnwidth]{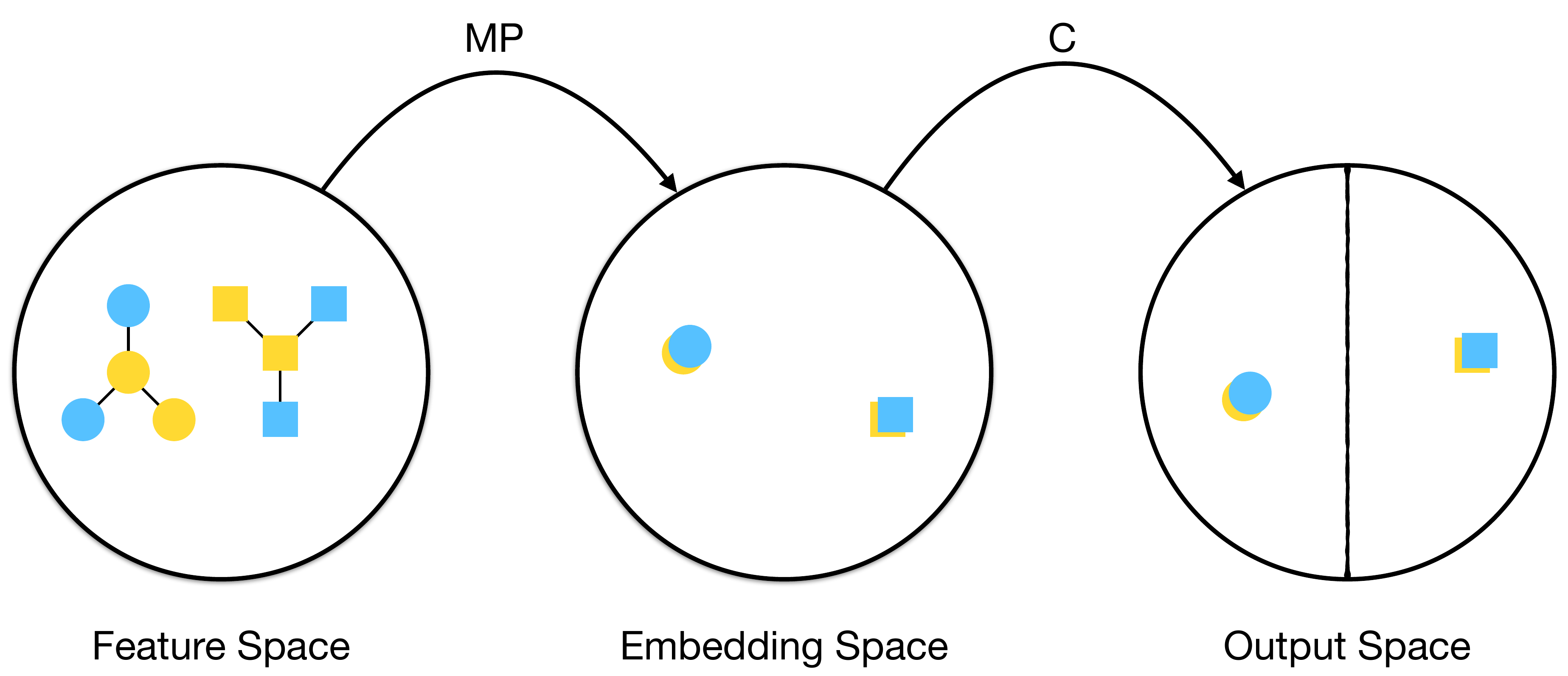}
        \caption{The graph-structured data in an ideal scenario. Shapes represent node categories, while lines between nodes indicate adjacency relationships. Blue denotes training nodes, and yellow signifies test nodes. MP stands for Message Passing, and C denotes the Trainable Classifier.}
        \label{fig1}
    \end{figure}

    To better articulate our rationale for selecting $X^*$, let us first illustrate with a toy example. Consider an ideal graph-structured data, as shown in Figure \ref{fig1}. In the figure, shapes represent node categories, while lines between shapes signify adjacency relationships between nodes. Blue indicates training nodes, and yellow denotes test nodes. We consider an extreme scenario. In this scenario, nodes of the same category share identical features. The feature disparity between nodes of different categories is pronounced. Furthermore, connections are present only between nodes of the same category. After undergoing Message Passing, the embeddings of all nodes of the same category in the embedding space become consistent, while nodes of different categories map to distinct regions. Such a setup can yield excellent node classification results, attributable to the following two reasons:
    \begin{itemize}
        \item Embeddings of nodes from different categories in the embedding space are separable.
        \item Embeddings of test nodes and training nodes in the latent space are identical.
    \end{itemize}

    While real-world graph data may hardly meet such assumptions, a well-trained GNN should be adept at classifying embeddings of training nodes in the embedding space. Hence, for such a well-trained GNN, we should be able to identify a representative embedding for each category from the embeddings of training nodes or their linear combinations. Such embedding should be easily classifiable. If the embeddings of test nodes is close to the representative embeddings of their respective categories and these representative embeddings for different categories are separable, test nodes remain effectively classified. Consequently, our goal is to ensure that by replacing the features of labeled nodes with representative embeddings, the embeddings of test nodes align closely with the representative embeddings of their categories.
    We introduce the definition of the class representative embedding.

    \begin{Definition}
        \textbf{Class Representative Embedding}. If the embedding $h_{c_i}^*$ of node $i$ with the label $c_i$ satisfies the following conditions:
        \begin{equation}
            \lVert C(h_{c_i}^*)- c_i \rVert _2\leq \epsilon,
        \end{equation}
        where $C(\cdot)$ denotes the trainable classifier in GNNs and $\epsilon$ denotes a small positive number. Then, $h_{c_i}^*$ is the class representative embedding of class $c_i$.
    \end{Definition}

    Initially, we prove that the embedding of test nodes will center around the embedding expectation of nodes in its respective category. We make the following assumptions:

    \begin{Assumption}
        \label{ass1}
        The node feature $x_i$ of the node $i$ in graph data follows $x_i\sim F_{c_i}$, where $c_i$ denotes the category to which node $i$ belongs and $F_{c_i}$ represents the feature distribution of nodes in category $c_i$. The labels of the neighbor nodes of node $i$ are independently sampled $deg(i)$ times from the distribution $D_{c_i}$, where $D_{c_i}$ is the label distribution of the neighbors of node $i$ and $deg(i)$ is the degree of node $i$.
    \end{Assumption}

    Under the aforementioned assumptions, we conduct a theoretical analysis on the Message Passing mechanism. The node embedding $h_i$ after Message Passing can be represented by the following equation:
    \begin{equation}
            h_i=\sum_{j\in N_i\cup \{i\}}a_{ij}x_j,
    \end{equation}
    where $a_{ij}$ denotes the weight of the edge between node $i$ and node $j$ and $N_i$ denotes the set of neighboring nodes for node $i$. Using this formula, we can obtain the embedding expectation $e_{c_i}$ of class $c_i$. We present the expression for $e_{c_i}$ in Lemma \ref{lemma1}.

    \begin{Lemma}
        \label{lemma1}
        Under Assumption 1, the embedding expectation $e_{c_i}$ of class $c_i$ can be represented as:
        \begin{equation}
                e_{c_i}=\sum_{j\in N_i\cup \{i\}}a_{ij}\cdot \mathbb{E}_{c\sim D_{c_i}, x\sim F_{c}}[x],
        \end{equation}
    \end{Lemma}
    \begin{proof}
        Please refer to Appendix A.
    \end{proof}

    Building upon Lemma \ref{lemma1}, we aim to use the embedding expectation of each category to analyze the relationship between the representation of any given node and the embedding expectation of its respective category. We present the following theorem:

    \begin{Theorem}
        \label{th1}
        Under Assumption 1, for any node $i$ with the label $c_i$, the relationship between its embedding $h_i$ in the embedding space and the embedding expectation $e_{c_i}$ of its category is given by:
        \begin{equation}
            \label{eq1}
            P(\lVert h_i - e_{c_i}\rVert\geq t)\leq 2l\cdot \exp(-\frac{t^2}{2\sigma^2\cdot l\cdot \text{deg}(i)\cdot \gamma^2})
        \end{equation}
        where $t$ denotes a positive number, $l$ denotes the dimension of node features, $\sigma$ denotes the upper bound of any dimension of node features, $\text{deg}(i)$ denotes the degree of node $i$, and $\gamma$ denotes the maximum weight of the edge between node $i$ and node $j$.
    \end{Theorem}
    \begin{proof}
        Please refer to Appendix B.
    \end{proof}

    Theorem \ref{th1} demonstrates that for graph neural networks satisfying Assumption \ref{ass1}, the node embeddings derived from Message Passing are close to the embedding expectation of their respective categories with high probability. This embedding expectation is correlated with both the node feature distribution and the label distribution of neighbor nodes of that category. 

    Furthermore, we prove that by replacing the features of labeled nodes with class representative embeddings, the embedding expectation will draw closer to the class representative embeddings of its category. To substantiate this theorem, we present the following lemma:

    \begin{Lemma}
        \label{lemma2}
        Let the set $\mathcal{A}=\{a_1,a_2,\dots,a_n\}$, where $a$ is independently and identically distributed. We randomly select $m$ elements from $A$ and replace them with $b$. Let $\epsilon$ be a small positive number. If $\lVert \sum_{i=1}^{n-m}a_i-(n-m)\cdot \mu(\mathcal{A})\rVert_2 \leq \epsilon$, the new set $\mathcal{B}$ satisfies:
        \begin{equation}
            \begin{aligned}
                 \lVert \mu(\mathcal{A})-y\rVert_2 - \lVert \mu(\mathcal{B})-y\rVert_2 \geq m \cdot \lVert \mu(\mathcal{A})-y\rVert_2-\epsilon
            \end{aligned}
        \end{equation}
        where $\mu(\cdot)$ denotes the mean operater.
    \end{Lemma}
    \begin{proof}
        Please refer to Appendix C.
    \end{proof}

    Utilizing the Lemma \ref{lemma2}, we can prove the following theorem.

    \begin{Theorem}
        After replacing the labeled nodes with the class representative embedding of the same category, the relationship between the embedding expectation $e_c^*$ of category $c$ and the class representative embedding $h_c^*$ of category $c$ is given by:
        \begin{equation}
            \label{eq2}
            \lVert e_{c}^* - h_{c}^*\rVert \leq \lVert e_{c} - h_{c}^*\rVert
        \end{equation}
        where $e_{c}$ denotes the embedding expectation of category $c$ without such replacement.
    \end{Theorem}
    \begin{proof}
        Please refer to Appendix D.
    \end{proof}

    Theorem 1 demonstrates that the embeddings of test nodes derived by Message Passing tend to center around the embedding expectation of their category. Theorem 2 proves that after replacing the features of labeled nodes with the class representative embeddings of the same category, the embedding expectations draw closer to the class representative representation of their corresponding category. Combining Theorems 1 and 2, it is evident that after substituting the features of labeled nodes with the class representative embeddings, the distance between the embeddings of test nodes and the class representative embeddings of their category is reduced. Since the class representative embeddings are selected from the embeddings of training nodes or their interpolations, the bias between the embeddings of test nodes and training nodes is diminished. We have addressed the first question.$X^*$ is the node feature matrix obtained by replacing features of labeled nodes of the same category in $X$ with the class representative embeddings.

\subsection{How can $X^*$ be found without altering the well-trained GNN model?}
    In the previous section, we proved that by replacing the features of labeled nodes with class representative embeddings, we can reduce the bias between the embeddings of test nodes and training nodes. This reduction can lead to improved test performance. To identify such class representative embeddings for feature reconstruction, we introduce a feature reconstruction framework for GNNs. This framework maps genuine labels back to the feature space to pinpoint the class representative embeddings that can be most easily classified into their corresponding labels. Notably, our framework neither necessitates modifications to the existing structure nor demands adjustment of the trained parameters. Consequently, our framework ensures both interpretability and portability.

    \begin{figure}[]
        \centering
        \includegraphics[width=1\columnwidth]{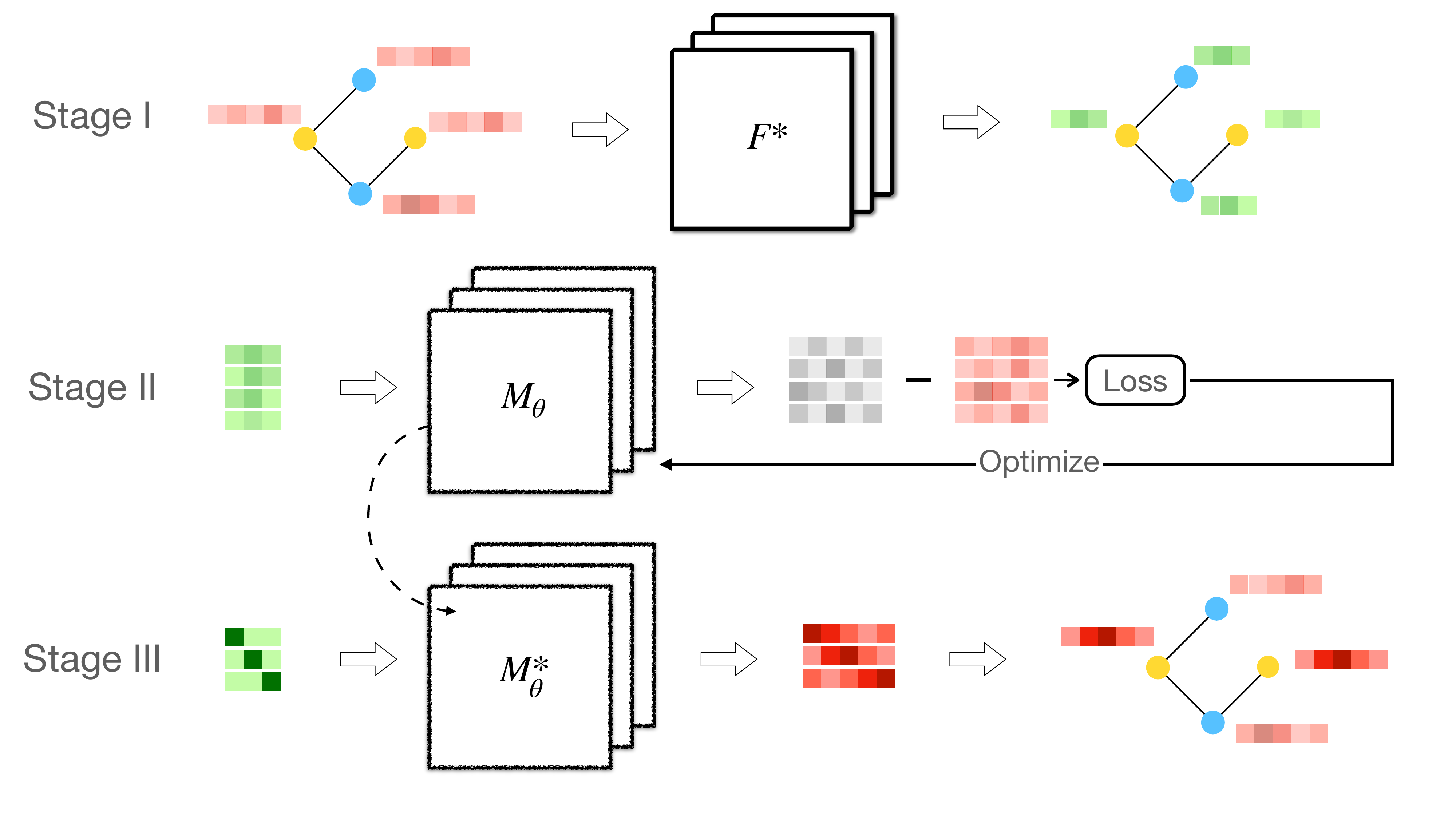}
        \caption{The overall architecture of our framework. The yellow nodes in Stage I denote the labeled nodes, while the blue nodes denote the unlabeled nodes.}
        \label{FRGNN}
    \end{figure}

    The overall architecture of our feature reconstruction framework is depicted in Figure \ref{FRGNN}. Initially, a GNN model $F^*$ is trained following the standard procedure of the selected GNN. Subsequently, we can obtain the prediction $\hat{Y}$ of all nodes with the trained model $F^*$.
    \begin{equation}
            \hat{Y} = F^*(X,A),
    \end{equation}
    where $X$ denotes the node feature matrix and $A$ denotes the adjacency matrix.
    Following that, we employ a Multi Layer Perceptron (MLP) to learn the mapping between the predictions and the features. Formally, we have:
    \begin{equation}
            \hat{X} = M_{\theta}(\hat{Y}),
    \end{equation}
    where $\hat{X}$ denotes the predicted features and $M_{\theta}$ denotes the MLP with parameters $\theta$. We choose the mean square error (MSE) as the loss function for the MLP. The loss function is given by:
    \begin{equation}
            L(X, \hat{X}) = \frac{1}{2}\lVert X-\hat{X}\rVert _2^2.
    \end{equation}
    By optimizing the loss function, we obtain the optimal parameters $\theta^*$. Feeding the ground-truth label of each category of nodes into this MLP, the output provides the class representative embedding for that category.
    \begin{equation}
            x_c^* = M_{\theta^*}(y_c),
    \end{equation}
    where $y_c$ denotes the one-hot vector of category $c$ and $x_c^*$ denotes tthe class representative embedding.
    Finally, we replace the features of the labeled nodes with the class representative embedding of their respective categories. The reconstruction feature matrix $X^*$ can be directly utilized for testing the well-trained model $F^*$. The pseudocode for the feature reconstruction process is delineated in Algorithm \ref{alg1}.
    
    \begin{algorithm}[]
        \label{alg1}
        \caption{FR-GNN Framwork}
        \textbf{Input}: Graph $G=(A,X)$, GNN Model $F^*$ trained on $G$\\
        \textbf{Output}: Class representative features $X^*$\\
        \begin{algorithmic}[1]
        \STATE $\hat{Y}$ = $F^*(X, A)$; // $\hat{Y}$ is the prediction of all nodes;
        \STATE Randomly initialize an MLP;
        \STATE Train the MLP on $\hat{Y}$ and $X$ and the learned parameter of the MLP is $\theta^*$;
        \FOR{$c=1$ to $C$}
        \STATE $x_c^* = M_{\theta^*}(y_{c})$, // $y_c$ is the one-hot vector of class $c$;
        \ENDFOR
        \STATE Replace the labeled node features with the class representative features $x_c^*$ and obtain the reconstructed feature $X^*$;
        \STATE \textbf{return} $X^*$
        \end{algorithmic}
    \end{algorithm}

    To ascertain that our framework accurately identifies the appropriate class representative embeddings, we will demonstrate that the outputs obtained via the aforementioned feature reconstruction approach indeed correspond to the desired class representative embeddings. We present Lemma \ref{lemma3} and subsequently employ it to establish the proof of Theorem \ref{th3}.

    \begin{Lemma}
        \label{lemma3}
        For a MLP with $l$ layers and a ReLU activation function, if the input satisfies $\lVert x_1 - x_2 \rVert _2^2 \leq \epsilon$, then the output satisfies $\lVert MLP(x_1) - MLP(x_2) \rVert _2^2 \leq \delta$.
    \end{Lemma}
    \begin{proof}
        Please refer to Appendix E.
    \end{proof}

    \begin{Theorem}
        \label{th3}
        If $\lVert MLP(\hat{Y})-X \rVert _2^2 \leq \epsilon$ and $x_c^*=MLP(y_c)$, then $\lVert C(x_c^*) - y_c \rVert _2^2 \leq \delta$, where $C(\cdot)$ denotes a trainable classifier in GNNs, $x_c^*$ denotes the class representative embedding of category $c$, and $y_c$ denotes the one-hot vector of category $c$.
    \end{Theorem}
    \begin{proof}
        Please refer to Appendix F.
    \end{proof}

    Theorem 3 establishes that the outputs derived from our framework indeed correspond to the desired class representative embeddings.

\section{Experiments}
    In this section, we evaluate the performance of the proposed FR-GNN framework in the context of semi-supervised node classification tasks. Comprehensive experiments are conducted across several publicly available datasets to validate the efficacy of our framework. Moreover, to demonstrate the portability of our framework, we integrate it with various foundational GNN models. Specifically, our experiments aim to address the following three research questions:
    \begin{itemize}
        \item  How does FR-GNN perform in the context of semi-supervised node classification tasks?
        \item After feature reconstruction, can we observe a reduction in the embedding bias between training and test nodes?
        \item Which nodes are correctly classified as a direct consequence of the feature reconstruction?
    \end{itemize}

    \begin{table}[]
        \centering
        \caption{Overall datasets statistics}
        \begin{tabular}{cccccc}
        \hline
        Dataset & Cora & Citeseer & Pubmed & ogb-arxiv\\ \hline
        \#Nodes & 2,708 & 3,327 & 19,717 & 169,343\\
        \#Edges & 5,429 & 4,732 & 44,338 & 1,166,243\\
        \#Features & 1,433 & 3,703 & 500 & 128\\
        \#Classes & 7 & 6 & 3 & 40\\
        \#Label Rate & 5.2\% & 3.6\% & 0.3\% & 5.0\%\\ \hline
        \#Train Nodes & 140 & 120 & 60 & 8,467\\
        \#Val Nodes & 500 & 500 & 500 & 29,799\\
        \#Test Nodes & 1,000 & 1,000 & 1,000 & 48,603\\ \hline
        \label{tab2}
        \end{tabular}
    \end{table}

    \begin{table*}[]
        \centering
        \caption{Results(\%) for the semi-supervised node classification task on the four datasets. Training Split represents the method for selecting training nodes. 'Random' means that training nodes are chosen randomly, while 'Bias' means that training nodes are selected using the Scalable Biased Sampler~\cite{c:20}. Biased training samples simulate the situation with distribution shift. Each result is reported as average $\pm$ standard deviation across 100 experiments(10 different biased training sets $\times$ 10 random initializations). \textbf{OOM} stands for Out of Memory. The best results are highlighted in bold. The symbol $*$ indicates that the performance after utilizing the framework surpassed that of the base model.}
        \begin{tabular}{lclllll}
        \hline
        Method & Training split & Cora & Citeseer & Pubmed & ogb-arxiv\\ \hline
        GCN & Random & 80.92 $\pm$ 0.78 & 70.93 $\pm$ 0.84 & 79.31 $\pm$ 0.53 & 69.15 $\pm$ 0.63\\ \hline
        GCN & Bias & 69.24 $\pm$ 1.51 & 62.23 $\pm$ 1.21 & 62.93 $\pm$ 3.23 & 65.84 $\pm$ 0.82\\
        FR-GCN & Bias & 74.12 $\pm$ 1.44$^*$ & 63.47 $\pm$ 1.18$^*$ & \textbf{65.79 $\pm$ 3.53}$^*$ & \textbf{67.35 $\pm$ 0.62}$^*$\\ \hline
        GAT & Bias & 72.24 $\pm$ 2.14 & 63.56 $\pm$ 1.28 & 60.81 $\pm$ 3.98 & 63.45 $\pm$ 0.86\\
        FR-GAT & Bias & 76.84 $\pm$ 2.03$^*$ & 66.49 $\pm$ 1.36$^*$ & 62.59 $\pm$ 3.75$^*$ & 65.32 $\pm$ 0.73$^*$\\ \hline
        GraphSage & Bias & 67.10 $\pm$ 4.46 & 61.00 $\pm$ 1.43 & 60.60 $\pm$ 4.24 & 62.46 $\pm$ 0.77\\
        FR-GraphSage & Bias & 70.98 $\pm$ 3.98$^*$ & 66.65 $\pm$ 0.89$^*$ & 61.06 $\pm$ 4.01$^*$ & 62.70 $\pm$ 0.76$^*$\\ \hline
        APPNP & Bias & 73.31 $\pm$ 2.98 & 64.14 $\pm$ 2.58 & 63.31 $\pm$ 3.52 & 65.58 $\pm$ 0.73\\
        FR-APPNP & Bias & \textbf{78.11 $\pm$ 2.60}$^*$ & \textbf{66.73 $\pm$ 2.11}$^*$ & 65.70 $\pm$ 3.57$^*$ & 66.20 $\pm$ 0.51$^*$\\ \hline
        GAUG & Bias & 73.05 $\pm$ 2.30 & 66.02 $\pm$ 2.29 & \textbf{OOM} & \textbf{OOM}\\
        MH-AUG & Bias & 75.51 $\pm$ 3.59 & 60.58 $\pm$ 3.42 & 57.73 $\pm$ 3.55 & 55.78 $\pm$ 1.23\\ 
        KDGA & Bias & 75.40 $\pm$ 3.52 & 61.58 $\pm$ 2.96 & \textbf{OOM} & \textbf{OOM}\\
        WT-AWP & Bias & 74.92 $\pm$ 2.62 & 65.72 $\pm$ 2.30 & 63.79 $\pm$ 4.23 & \textbf{OOM}\\
        SR-GNN & Bias & 76.32 $\pm$ 2.23 & 66.26 $\pm$ 1.47 & 65.38 $\pm$ 3.96 & 66.50 $\pm$ 0.60\\ \hline
        \label{tab1}
        \end{tabular}
    \end{table*}

    \subsection{Experimental Settings}
        \subsubsection{Base Models}
        The framework we propose requires a base model to function effectively. In this context, we utilize GCN~\cite{c:1}, GAT~\cite{c:2}, GraphSAGE~\cite{c:26}, and APPNP~\cite{c:28} as our base models. By prefixing 'FR' to the name of the base model, we designate the base model that has incorporated our framework. For instance, when the Base Model is GCN, we refer to it as FR-GCN.

        \subsubsection{Baselines}
        To validate the effectiveness of our framework, we compare it with three categories of methods wihch are commonly employed to address the issue of distribution shift in GNNs. The first category comprises the base models. The second category consists of robust GNNs, which are designed to make GNN models robust to distribution shift. The third category comprises the data augmentation methods, which are designed to increase the diversity and quality of training data. Specifically, we compare our framework with the following methods:
        \begin{itemize}
            \item Base models: GCN~\cite{c:1}, GAT~\cite{c:2} and GraphSAGE~\cite{c:26};
            \item Robust GNN: WT-AWP~\cite{c:21}, SR-GNN~\cite{c:20};
            \item Data Augmentation methods: GAUG~\cite{c:22}, MH-AUG~\cite{c:23}, KDGA~\cite{c:24}.
        \end{itemize}

        \begin{figure}[]
            \centering
            \includegraphics[width=1\columnwidth]{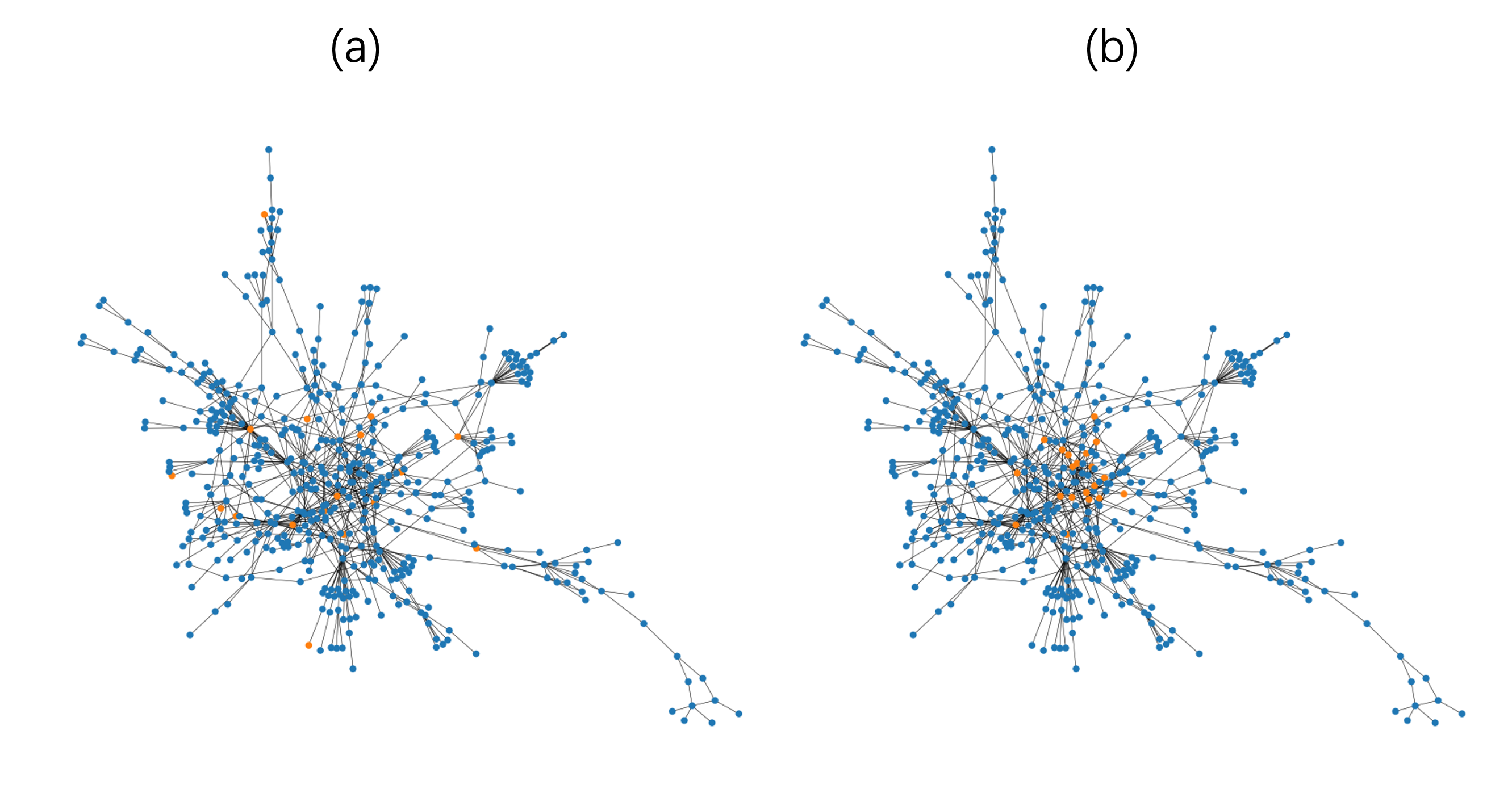}
            \caption{Visualization of a subgraph of Cora for a specific class. (a) Randomly selected training sample; (b) Biased training sample.}
            \label{fig6}
        \end{figure}

        \begin{figure}[]
            \centering
            \includegraphics[width=1\columnwidth]{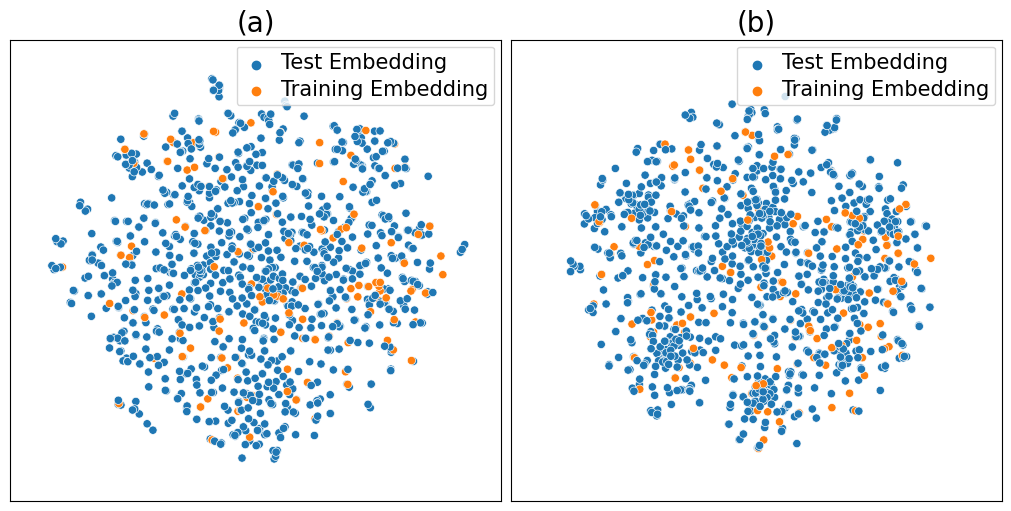}
            \caption{Visualization of the Cora dataset after dimensionality reduction with t-SNE. (a) Node embeddings derived from the original node features; (b) Node embeddings derived from the reconstructed node features. In each figure, orange points represent the training node embeddings, while blue points depict the test node embeddings.}
            \label{fig2}
        \end{figure}

        \subsubsection{Datasets}
        We conduct experiments on four popular benchmark datasets: Cora, Citeseer, Pubmed~\cite{c:27}, and ogb-arxiv~\cite{c:29}. For each dataset, we utilize the same validation and test splits as in the original papers~\cite{c:27,c:29}. Subsequently, we sample training nodes from the remaining nodes. The training nodes in the original papers are randomly selected, which could not capture distribution shift. In order to simulate distribution shift, we employed the Scalable Biased Sampler introduced in ~\cite{c:20} to acquire biased training nodes. This method utilizes the Personalized PageRank algorithm, allowing efficient sampling of biased training samples in large-scale datasets. We visualize a subgraph of Cora for a specific class in Figure \ref{fig6}, where the sampled training nodes are marked in orange. Figure \ref{fig6}(a) represents the randomly selected training nodes as in the original paper, while Figure \ref{fig6}(b) shows the biased training samples. The datasets statistics are presented in Table \ref{tab2}.

        \subsubsection{Implementation Details}
        We begin by training the base model on these datasets, subsequently using the well-trained model to obtain the prediction $\hat{y}$ for all nodes. Next, we use $\hat{y}$ as the input with the node feature $x_i$ serving as the label to train a 2-layer MLP. We employe the Adam optimizer with a learning rate of 0.001 and weight decay of 0.0005, running the training for 50 epochs. Upon completion of the training, we feed the ground truth of each class (one-hot) into the MLP. The output from the model serve as the representative embedding for each class. Subsequently, we replace the features of the training and validation nodes with the class representative embedding of their respective classes to obtain reconstructed features. Finally, we test on the well-trained GNN using this reconstructed features. Our framework is implemented with PyTorch, and all experiments were conducted on NVIDIA RTX 3090 24G GPU.

        \begin{figure}[]
            \centering
            \includegraphics[width=1\columnwidth]{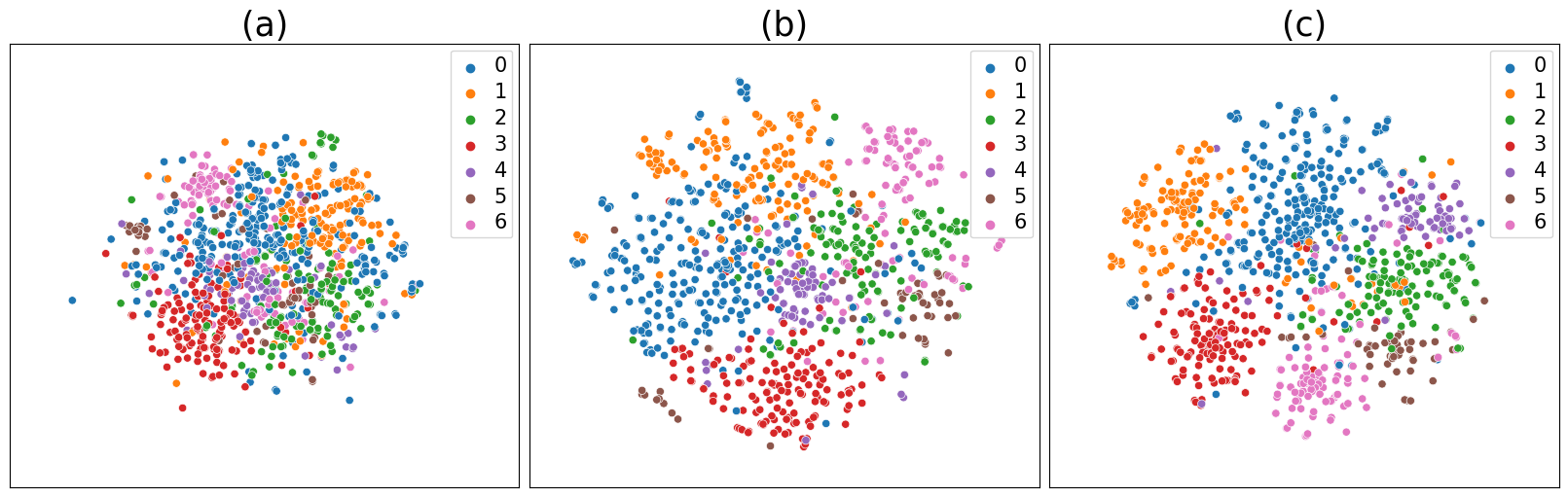}
            \caption{Visualization of the Cora dataset after dimensionality reduction with t-SNE. (a) Original node features; (b) Node embeddings derived from the original node features; (c) Node embeddings derived from the reconstructed node features. Different colored points indicate different categories.}
            \label{fig3}
        \end{figure}

    \subsection{How does FR-GNN perform in the context of semi-supervised node classification tasks?}
        The experimental results of our framework integrated with different base models are presented in Table \ref{tab1}. Across all datasets, our framework consistently improve the test performance of the base models. By employing feature reconstruction and leveraging the unique message-passing mechanism of GNNs, our framework modifies the test embedding distribution and mitigate the embedding bias between test and training nodes. The experimental results validate the efficacy and superiority of our framework. Additionally, the results underscore the portability and flexibility of our framework. Notably, our framework is constructed based on the generic representation paradigm of GNNs rather than being tailored to a particular base model.

        Table \ref{tab1} also presents a comparison between our framework and various baselines. Our framework outperforms all baselines on the Cora and Citeseer datasets. On the Pubmed dataset, our framework surpasses all baselines except for WT-AWP. We observe a less pronounced improvement across all methods on the Pubmed dataset compared to Cora and Citeseer. One primary reason for this observation is the reduced number of labeled nodes in the Pubmed dataset. Specifically, the proportion of labeled nodes in Pubmed is notably lower than in the other two datasets. This scarcity in labeled data inherently makes the semi-supervised node classification task more challenging for Pubmed. Our framework relies on the message-passing mechanism, aiming to optimize the embeddings of unlabeled nodes using those of labeled nodes. However, as the proportion of labeled nodes decreases, this optimization effect diminishes.

    \subsection{After feature reconstruction, can we observe a reduction in the embedding bias between training and test nodes?}

        We conducted a visual analysis of the node embedding derived from original features and reconstructed features of the Cora dataset. Figure \ref{fig2} visualizes the node embeddings with TSNE. Comparing subfigures (a) and (b), we can observe that after feature reconstruction, the embeddings of test nodes become more compact and closer to the embeddings of training nodes. This validates our Theorem 2 that feature reconstruction can reduce the embedding bias between the training and test nodes. Similarly, by contrasting subfigures (c) and (d), we can draw the same conclusion.

        Figure \ref{fig3} visually displays the distribution characteristics of nodes from different classes. As can be observed from subfigure 3(a), the original node features exhibit a relatively scattered distribution for each class. There is considerable overlap between features of different classes. Subfigure 3(b) showcases the node embeddings derived with Message Passing. It is evident that the node embeddings for each class are no longer dispersed, which validates our Theorem \ref{th1} that the embeddings of nodes from different classes tend to center around the average embedding of their respective class. Subfigure 3(c) presents the embeddings after feature reconstruction. As depicted, the node embeddings derived from reconstructed features have a more compact distribution with lesser overlap between features of different classes. Such characteristics are more conducive to classification.

    \subsection{Which nodes are correctly classified as a direct consequence of the feature reconstruction?}
        To further investigate the application scenarios of our framework, we analyze the test results after feature reconstruction. Specifically, we identify nodes that transitioned from misclassified to correctly classified due to feature reconstruction.

        \begin{figure}[]
            \centering
            \includegraphics[width=1\columnwidth]{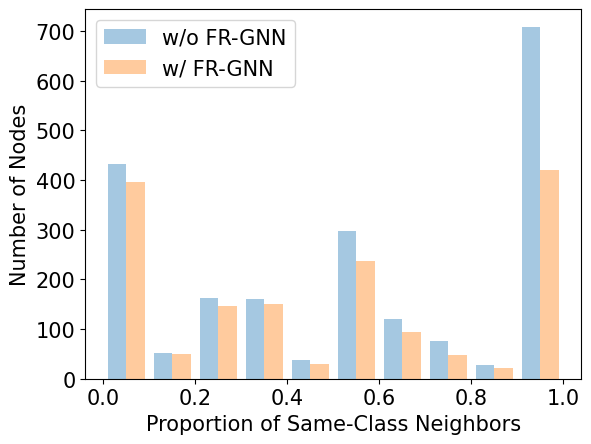}
            \caption{Number of Misclassified Nodes with Different Proportions of Same-Class Neighbors on Cora dataset. The count is the sum of the results from 10 experiments.}
            \label{fig4}
        \end{figure}

        We conducted 10 random experiments on the Cora dataset with GCN as the base model. The experimental results are displayed in Figure 4. We observe that when the proportion of same-class nodes in the neighboring nodes is high, the number of misclassified nodes reduces dramaticly after using our framework. Moreover, the majority of the performance improvement comes from these nodes. On the other hand, when the proportion of same-class nodes in the neighboring nodes is low, our framework provides limited assistance to these nodes. This observation aligns with intuition. When same-class nodes occupy a significant portion of neighbor nodes, there is a higher probability that the class representative representation directly or indirectly improves the embeddings of the same-class node. Conversely, when the proportion of same-class nodes among the neighbors is low, the positive influence is limited.

\section{Conclusion}
	In this paper, we propose a feature reconstruction framework for GNNs. Our framework can mitigate the adverse effects of distribution shift without altering the model structure or trained parameters. Specifically, our framework modifies the features of labeled nodes to influence the embeddings of other nodes via message passing mechanism at test time. Such test-time modifications exhibit excellent portability and flexibility. Subsequently, we conduct an in-depth analysis of the theoretical reason for the effectiveness of our framework. We also detail the specific implementation method for our framework. Moreover, we conduct comprehensive experiments on several publicly available datasets to validate the efficacy of our framework. The experimental results demonstrate that our framework can effectively mitigate the distribution shift and improve the test performance. 

    However, our framework has its limitations. Our framework relies on labeled nodes for feature reconstruction. If the labeled nodes are inaccessible or insufficient, our framework might fail to improve the test performance effectively. A potential avenue for future work would be determining whether it is feasible to achieve feature reconstruction during the testing phase solely based on node features.

\appendices
\section{Proof of Lemma 1} \label{Appendix A}
\begin{proof}
	For any node $i$ with the label $c_i$, $h_i$ denotes its embedding derived from message passing. We assume that the node feature is independently and identically distributed from the distribution $F_c$ when the node is with the label $c$. Then the embedding expectation $e_{c_i}$ of class $c_i$ can be represented as:
        \begin{equation*}
            \begin{aligned}
                e_{c_i} &= \mathbb{E}[h_i] \\
                & = \mathbb{E}[\sum_{j\in N_i\cup \{i\}}a_{ij}\cdot x_j] \\
                & = \mathbb{E}[\sum_{j=1}^{\abs{V}}a_{ij}\cdot x_j\cdot \mathbb{I}(j\in N_i\cup \{i\})] \\
                & = \sum_{j=1}^{\abs{V}}a_{ij}\cdot \mathbb{E}[x_j\cdot \mathbb{I}(j\in N_i\cup \{i\})] \\
                & = \sum_{j\in N_i\cup \{i\}}a_{ij}\cdot \mathbb{E}_{c\sim D_{c_i}, x\sim F_c}[x], \\
            \end{aligned}
        \end{equation*}
        where $\mathbb{I}(\cdot)$ denotes the indicator function. When the condition $j\in N_i\cup \{i\}$ is satisfied, the indicator function $\mathbb{I}(j\in N_i\cup \{i\})$ equals 1, otherwise it equals 0.
\end{proof}

\section{Proof of Lemma 2} \label{Appendix B}
\begin{proof}
	\begin{equation*}
		\begin{aligned}
			& \lVert h_i - e_{c_i}\rVert \\
            &= \lVert \sum_{j\in N_i\cup \{i\}}a_{ij}\cdot (x_j -  \mathbb{E}_{c\sim D_{c_i}, x\sim F_{c}}[x])\rVert_2 \\
			& \leq \max\{a_{ij}\}\cdot \sum_{j\in N_i\cup \{i\}} \lVert x_j -  \mathbb{E}_{c\sim D_{c_i}, x\sim F_{c}}[x]\rVert_2 \\
		\end{aligned}
	\end{equation*}

	\begin{equation*}
		\begin{aligned}
			& P(\lVert h_i - e_{c_i}\rVert\geq t) \\
			& \leq P(\max\{a_{ij}\}\cdot \sum_{j\in N_i\cup \{i\}} \lVert x_j -  \mathbb{E}_{c\sim D_{c_i}, x\sim F_{c}}[x]\rVert_2\geq t)\\
			& = P(\sum_{j\in N_i\cup \{i\}} \lVert x_j -  \mathbb{E}_{c\sim D_{c_i}, x\sim F_{c}}[x]\rVert_2\geq \frac{t}{\max\{a_{ij}\}})\\
			& \leq P(\cup_{k=1}^{l}\{ \sum_{j\in N_i\cup \{i\}}\abs{x_j[k] - \\  
            &\mathbb{E}_{c\sim D_{c_i}, x\sim F_{c}}[x][k]} \geq \frac{t}{\sqrt{l}\cdot \max\{a_{ij}\}}\})\\
			& = \sum_{k=1}^{l}P(\sum_{j\in N_i\cup \{i\}}\abs{ x_j[k] - \\
            & \mathbb{E}_{c\sim D_{c_i}, x\sim F_{c}}[x][k]} \geq \frac{t}{\sqrt{l}\cdot \max\{a_{ij}\}})\\
			&
		\end{aligned}
	\end{equation*}

	We use Hoeffding's inequality to bound the probability. Hoeffding's inequality can be expressed as:
	\begin{equation*}
		P(\sum_{i=1}^{n}\abs{z_i-\mathbb{E}[z_i]}\geq s)\leq \exp(-\frac{2s^2}{n(b-a)^2}),
	\end{equation*}
	where $z_i$ is a random variable, $a\leq z_i\leq b$, and $s>0$.

	Then we can bound the probability as:
	\begin{equation*}
		P(\lVert h_i - e_{c_i}\rVert\geq t) \leq 2l\cdot \exp(-\frac{t^2}{2\sigma^2\cdot l\cdot \text{deg}(i)\cdot \gamma^2}),
	\end{equation*}
	where $\sigma$ denotes the upper bound of any dimension of node features, $\text{deg}(i)$ denotes the degree of node $i$, and $\gamma=\max\{a_{ij}\}$.
\end{proof}

\section{Proof of Theorem 2} \label{Appendix C}
\begin{proof}
	We use the $c$-th class representative embedding $h_c^*$ to replace the labeled nodes features with the category $c$. According to Lemma 1, the embedding expectation $e_c^*$ of category $c$ can be represented as:
	\begin{equation*}
		\begin{aligned}
			e_{c}^* &= \sum_{j\in N_i\cup \{i\}}a_{ij}\cdot \mathbb{E}_{c\sim D_{c_i}, x\sim F_{c}^*}[x]. \\
		\end{aligned}
	\end{equation*}

	We can use the mean of node features with the $c$-th category to approximate the expectation of node features with the $c$-th category. Then the embedding expectation $e_c^*$ of category $c$ can be represented as:
	\begin{equation*}
		\begin{aligned}
			e_{c}^* &= \sum_{j\in N_i\cup \{i\}}a_{ij}\cdot \sum_{k=1}^{C}p_k \mu(F_{c_k}^*), \\
		\end{aligned}
	\end{equation*}
	where $p_k$ denotes the probability of the neighbor node with the $k$-th category in expectation. Then we have:
	\begin{equation*}
		e_c^* - h_c^* = \sum_{j\in N_i\cup \{i\}}a_{ij}\cdot \sum_{k=1}^{C}p_k \mu(F_{c_k}^*) - h_c^*.
	\end{equation*}

	From Lemma 2, we have:
	\begin{equation*}
		\lVert \mu(F_c)-h_c^*\rVert_2 - \lVert \mu(F_c^*)-h_c^*\rVert_2 \geq m \cdot \lVert \mu(F_c)-h_c^*\rVert_2-\epsilon.
	\end{equation*}

	Due to the fact that $h_c^*$ is obtained from a small proportion of nodes with the $c$-th category, $h_c^*$ is usually not close to $\mu(F_c)$. And $\epsilon$ is close to 0. Therefore, we have:
	\begin{equation*}
		\lVert\mu(F_c)-h_c^*\rVert_2-\epsilon \geq 0.
	\end{equation*}

	Then we have:
	\begin{equation*}
		\lVert \mu(F_c)-h_c^*\rVert_2 - \lVert \mu(F_c^*)-h_c^*\rVert_2 \geq 0.
	\end{equation*}

	Therefore, we have:
	\begin{equation*}
		\begin{aligned}
			\lVert e_{c} - h_{c}^*\rVert &= \sum_{j\in N_i\cup \{i\}}a_{ij}\cdot \sum_{k=1}^{C}p_k \mu(F_{c_k}) - h_c^* \\
			& \geq \sum_{j\in N_i\cup \{i\}}a_{ij}\cdot \sum_{k=1}^{C}p_k \mu(F_{c_k}^*) - h_c^* \\
			& = \lVert e_{c}^* - h_{c}^*\rVert.
		\end{aligned}
	\end{equation*}
\end{proof}

\section{Proof of Lemma 3} \label{Appendix D}
\begin{proof}
	For a single layer MLP, we have:
	\begin{equation*}
		f(x)=\sigma(x\cdot W+B),
	\end{equation*}
	where $x\in \mathbb{R}^{l}$ denotes the input, $W\in \mathbb{R}^{l\times h}$ denotes the weight matrix, $B\in \mathbb{R}^{h}$ denotes the bias vector, and $\sigma(\cdot)$ denotes the ReLU activation function. For any two inputs $x_1$ and $x_2$, we have:
	\begin{equation*}
		\begin{aligned}
			\lVert f(x_1)-f(x_2)\rVert _2&=\lVert \sigma(x_1\cdot W+B)-\sigma(x_2\cdot W+B)\rVert _2\\
			&\leq \lVert x_1\cdot W+B-x_2\cdot W-B\rVert _2\\
			&(ReLU \ Lipschitz \ Condition)\\
			&=\lVert (x_1-x_2)\cdot W\rVert _2\\
			&\leq \lVert x_1-x_2\rVert _2\cdot \lVert W\rVert _2\\
			&(Cauchy-Schwarz\ Inequality).\\
		\end{aligned}
	\end{equation*}

	For a MLP with $l$ layers, we have:
	\begin{equation*}
		\begin{aligned}
			f^{(l)}(x)&=\sigma(h^{(l)}\cdot W^{(l)}+B^{(l)})\\
			F^{(L)}(x)&=f^{(L)}(f^{(L-1)}(\cdots f^{(1)}(x)\cdots)),
		\end{aligned}
	\end{equation*}
	where $h^{(l)}$ denotes the output of the $l$-th layer, $W^{(l)}$ denotes the weight matrix of the $l$-th layer, and $B^{(l)}$ denotes the bias vector of the $l$-th layer. For any two inputs $x_1$ and $x_2$, we have:
	\begin{equation*}
		\begin{aligned}
			&\lVert F^{(L)}(x_1)-F^{(L)}(x_2)\rVert _2\\
			&=\lVert f^{(L)}(f^{(L-1)}(\cdots f^{(1)}(x_1)\cdots))-\\
			&f^{(L)}(f^{(L-1)}(\cdots f^{(1)}(x_2)\cdots))\rVert _2\\
			&\leq \lVert f^{(L-1)}(\cdots f^{(1)}(x_1)\cdots)-f^{(L-1)}(\cdots f^{(1)}(x_2)\cdots)\rVert _2\\
			&\leq \lVert f^{(L-2)}(\cdots f^{(1)}(x_1)\cdots)-\\
			&f^{(L-2)}(\cdots f^{(1)}(x_2)\cdots)\rVert _2\cdot \lVert W^{(L-1)}\rVert _2\\
			&\leq \cdots\\
			&\leq \lVert x_1-x_2\rVert _2\cdot \prod_{l=1}^{L-1}\lVert W^{(l)}\rVert _2.\\
		\end{aligned}
	\end{equation*}

	Therefore, if the input satisfies $\lVert x_1 - x_2 \rVert _2^2 \leq \epsilon$, then the output satisfies $\lVert MLP(x_1) - MLP(x_2) \rVert _2^2 \leq \delta$, where $\delta=\epsilon\cdot \prod_{l=1}^{L-1}\lVert W^{(l)}\rVert _2$.
\end{proof}

\section{Proof of Theorem 3} \label{Appendix E}
\begin{proof}
	Let $h_c^*$ denote the class representative embedding of category $c$. We have:
	\begin{equation*}
		\lVert C(h_c^*) - y_c \rVert _2^2 \leq \eta,
	\end{equation*}
	where $C(\cdot)$ denotes a trainable classifier in GNNs and $\eta$ is a small positive number. We assume that there exists $x$ in the node feature space which satisfies:
	\begin{equation*}
		C(x)=\hat{y},\lVert C(x) - y_c \rVert _2^2 \leq \gamma,
	\end{equation*}
	where $\gamma$ is a small positive number. Then we have:
	\begin{equation*}
		\lVert C(x) - C(h_c^*) \rVert _2^2 \leq \eta + \gamma.
	\end{equation*}

	Based on Lemma 3,
	\begin{equation*}
		\lVert x - h_c^* \rVert _2^2 \leq \eta + \gamma.
	\end{equation*}

	When the MLP is well trained and $\lVert x-MLP(\hat{y})\rVert \leq \epsilon$. Therefore, we have:
	\begin{equation*}
		\lVert MLP(\hat{y}) - h_c^* \rVert _2^2 \leq \eta + \gamma + \epsilon.
	\end{equation*}

	Moreover, we can use Lemma 3 and $\lVert \hat{y}-y_c\lVert _2^2\leq \gamma$ to obtain:
	\begin{equation*}
		\lVert MLP(\hat{y}) - MLP(y_c) \rVert _2^2 \leq \theta,
	\end{equation*}
	where $\theta$ is a small positive number. Therefore, we have:
	\begin{equation*}
		\lVert MLP(y_c) - h_c^* \rVert _2^2 \leq \eta + \gamma + \epsilon + \theta.
	\end{equation*}

	When $\eta + \gamma + \epsilon + \theta$ is small enough, we can obtain:
	\begin{equation*}
		\lVert x_c^* - h_c^*\rVert _2^2 \leq \delta,
	\end{equation*}
	where $\delta$ is a small positive number. Based on Lemma 3, we can obtain the conclusion:
	\begin{equation*}
		\lVert C(x_c^*) - y_c \rVert _2^2 \leq \delta.
	\end{equation*}
\end{proof}


\bibliographystyle{IEEEtran}
\bibliography{ijcai22}

\vfill

\end{document}